\documentclass[11pt]{article}
\usepackage[margin=1.25in]{geometry}
\usepackage{times}
\usepackage{amsmath,amsfonts,amssymb}
\usepackage{mathtools}
\usepackage{enumerate}
\usepackage{verbatim}
\usepackage{commath}
    \usepackage[usenames]{color}
    \definecolor{plum}  {rgb}{.4,0,.4}
    \definecolor{BrickRed} {rgb}{0.6,0,0}
	\definecolor{DarkBlue} {rgb}{0,0,0.6}
\usepackage[plainpages=false,pdfpagelabels,colorlinks=true,linkcolor=BrickRed,citecolor=plum,urlcolor=DarkBlue]{hyperref}
\usepackage[round]{natbib}
\usepackage[mathcal]{euscript}
\usepackage{cleveref}

\def\ddefloop#1{\ifx\ddefloop#1\else\ddef{#1}\expandafter\ddefloop\fi}

\def\ddef#1{\expandafter\def\csname b#1\endcsname{\ensuremath{\boldsymbol{#1}}}}
\ddefloop ABCDEFGHIJKLMNOPQRSTUVWXYZabcdefghijklmnopqrstuvwxyz\ddefloop

\def\ddef#1{\expandafter\def\csname c#1\endcsname{\ensuremath{\mathcal{#1}}}}
\ddefloop ABCDEFGHIJKLMNOPQRSTUVWXYZ\ddefloop
 
\def\ddef#1{\expandafter\def\csname s#1\endcsname{\ensuremath{\mathsf{#1}}}}
\ddefloop ABCDEFGHIJKLMNOPQRSTUVWXYZ\ddefloop

\def\Reals{{\mathbb R}}
\def\Complex{{\mathbb C}}
\def\p{{\partial}} 
\def\Ex{{\mathbf E}} 
\def\Tr{{\mathsf T}} 

\DeclareMathOperator*{\argmin}{arg\,min}

\newenvironment{proof}{\paragraph{Proof:}}{\hfill$\square$}
\newtheorem{theorem}{Theorem}
\newtheorem{lemma}[theorem]{Lemma}

\title{\makebox[0pt]{Rademacher Complexity of Neural ODEs via Chen--Fliess Series}}
\author{Joshua Hanson \\ \href{mailto:jmh4@illinois.edu}{jmh4@illinois.edu} \\ 
\and Maxim Raginsky \\ \href{mailto:maxim@illinois.edu}{maxim@illinois.edu} }
\date{}

\begin{document}
\maketitle

\begin{abstract}

We show how continuous-depth neural ODE models can be framed as single-layer, infinite-width nets using the Chen--Fliess series expansion for nonlinear ODEs. In this net, the output ``weights'' are taken from the signature of the control input --- a tool used to represent infinite-dimensional paths as a sequence of tensors --- which comprises iterated integrals of the control input over a simplex. The ``features'' are taken to be iterated Lie derivatives of the output function with respect to the vector fields in the controlled ODE model. The main result of this work applies this framework to derive compact expressions for the Rademacher complexity of ODE models that map an initial condition to a scalar output at some terminal time. The result leverages the straightforward analysis afforded by single-layer architectures. We conclude with some examples instantiating the bound for some specific systems and discuss potential follow-up work.

\end{abstract}

\section{Introduction}

Several recent works have examined continuous-depth idealizations of deep neural nets, viewing them as continuous-time ordinary differential equation (ODE) models with either fixed or time-varying parameters. Traditional discrete-layer nets can be recovered by applying an appropriate temporal discretization scheme, e.g., the Euler or Runge-Kutta methods. In applications, this perspective has resulted in advantages concerning regularization \citep{kelly_2020, pal_2021, kobyzev_2021}, efficient parameterization \citep{queiruga_2020}, convergence speed \citep{chen_2023}, applicability to non-uniform data \citep{sahin_2019}, among others. As a theoretical tool, continuous-depth idealizations have lead to better understanding of the contribution of depth to model expressiveness and generalizability \citep{massaroli_2020, marion_2023}, new or improved training strategies via framing as an optimal control problem \citep{corbett_2022}, and novel model variations \citep{jia_2019, peluchetti_2020}.

Considered as generic control systems, continuous-depth nets can admit a number of distinct input-output configurations depending on how the control system ``anatomy'' is delegated. Controlled neural ODEs \citep{kidger_2020} and continuous-time recurrent neural nets \citep{fermanian_marion_2021} treat the (time-varying) control signal as the input to the model; the initial condition is either fixed or treated as a trainable parameter; the (time-varying) output signal is the model output; and any free parameters of the vector fields (weights) are held constant in time. One may instead consider the initial condition to be the input; the output signal at a fixed terminal time as the model output; and the (fixed or time-varying) control signal as a representative for (depth-varying) model parameters, which yields a typical neural ODE \citep{chen_2018}. Here the input is a static finite-dimensional vector rather than a sequence or function of time, and this is the setting we consider in this work. Recent work of \citet{marion_2023}, discussed below, also considers this setting.

New results and insight have emerged from studying neural nets in the infinite-width or ``mean-field'' limit \citep{lu_2020, jacot_2021}. We can apply similar methods to neural ODEs by representing them using an infinite series expansion. One such example is the Chen--Fliess series \citep{chen_1957,fliess_1981}, which represents the output of a system as a sum of iterated Lie derivatives of the output function multiplied by corresponding iterated integrals of the control input, eliminating any recursive dependence on the state. This sequence of iterated integrals of the control is called the signature, which has been used in rough path theory for approximating and reconstructing stochastic signals using finite-dimensional data \citep{fermanian_2023}, and has also appeared in the control literature as a tool for studying small-time asymptotic behavior of control trajectories \citep{sussman_1983}. Expressed as an infinite series using this formalism, the small-time initial-condition-to-output map is linear in the iterated integrals of the control input (i.e., elements of the input signature), and the remaining terms depend only on the initial condition.

The Chen--Fliess series can be interpreted as an infinitely wide, single-layer neural net where each node uses a different activation function computing the appropriate iterated Lie derivative in the series. In this net, the input weights are set to unity and the output weights are the corresponding signature elements. Applying series expansions for nonlinear ODEs in this way allows us to analyze continuous-depth nets using well-established techniques for single-layer, infinite-width nets, which are often comparatively simpler.
Our goal in this work is to demonstrate a compact generalization bound for neural ODEs using these techniques. \citet{marion_2023} also gives bounds on the generalization error of neural ODEs using covering number estimates for parameterized ODE model classes. By contrast, the generalization bounds in this work made use of Rademacher complexities and can complement those of \citet{marion_2023}.

The remainder of the paper is organized as follows. In Section \ref{sec:chen_fliess}, we describe the Chen--Fliess series and how it is used to generate a tractable model architecture. Section \ref{sec:main_result} defines the learning problem, then states the main Rademacher complexity bound and proof. In Section \ref{sec:examples} we give some concrete examples to instantiate the bound, with conclusions and discussion provided in Section \ref{sec:conclusion}.

\section{Chen--Fliess series}\label{sec:chen_fliess}

We are interested in maps $\varphi : \sX \subset \Reals^n \to \sY \subset \Reals$ that can be described by sending an initial condition $x_0 \in \sX$ to the resulting output $y(T) \in \sY$ at time $T$ of a fixed control system
\begin{equation}\label{eq:control_system_nonlinear}
\begin{split}
    \dot{x}(t) &= f(x(t),u(t)),\quad\quad x(0) = x_0 \\
    y(t) &= h(x(t),u(t))
\end{split}
\end{equation}
with a fixed control input $u : [0,T] \to \sU$, where $\sU$ is an arbitrary subset of some finite-dimensional vector space.

\subsection{Control-affine systems}

Consider the generic nonlinear control system \eqref{eq:control_system_nonlinear}. For the purposes of this work, we can restrict our attention to control-affine systems with linear outputs, which we will justify below. Assume that $f(\cdot,u) : \Reals^n \to \Reals^n$ is continuous for every $u \in \sU$. Then we can always find continuous vector fields $g_1,\dots,g_m : \Reals^n \to \Reals^n$ such that $\{f(x,u) : u \in \sU\} \subseteq \text{span}\{g_1(x),\dots,g_m(x)\}$ for every $x \in \sX$ \citep{sussmann_2008}. Now for each $x \in \Reals^n$ define
\begin{equation*}
    G := \begin{bmatrix} g_1(x) & \hdots & g_m(x) \end{bmatrix} \in \Reals^{n \times m},\quad\quad P := \begin{bmatrix} I & 0 \end{bmatrix} \in \Reals^{r \times m},
\end{equation*}
where $r \leq m$ is the rank of $G$. Assume without loss of generality that $g_{r+1}(x) = \cdots = g_m(x) = 0$, so then $GP^\Tr \in \Reals^{n \times r}$ has rank $r$. By construction, $f(x,u)$ is in the column space of $G$, thus for each $u \in \sU$ there exists $v \in \Reals^m$ such that $Gv = f(x,u)$. In this case, we can take
\begin{equation*}
    v = P^\Tr (P G^\Tr G P^\Tr)^{-1} P G^\Tr f(x,u).
\end{equation*}
Then the trajectory of the control-affine system
\begin{equation}\label{eq:control_system_affine1}
    \dot{x}(t) = \sum_{i=1}^m v_i(t) g_i(x(t)),\quad\quad x(0) = x_0
\end{equation}
is that same as the trajectory of \eqref{eq:control_system_nonlinear}. Considered as an open-loop control system, the set of admissible solutions of \eqref{eq:control_system_affine1} subsumes the set of admissible solutions of \eqref{eq:control_system_nonlinear}. The map $(x,u) \mapsto v$ may be discontinuous, but this is without consequence. From another angle that avoids state feedback, we could have instead considered appending the dynamics with an input integrator to yield
\begin{equation}\label{eq:control_system_affine2}
\begin{split}
    \dot{x}(t) &= f(x(t),u(t)),\quad\quad x(0) = x_0 \\
    \dot{u}(t) &= v(t),
\end{split}
\end{equation}
which is another control-affine system with the same trajectory as \eqref{eq:control_system_nonlinear}, provided that $u$ is at least weakly differentiable. Furthermore, differentiating the output yields the equation
\begin{equation*}
    \dot{y}(t) = \frac{\p h}{\p x}(x(t),u(t)) f(x(t),u(t)) + \frac{\p h}{\p u}(x(t),u(t)) v(t),
\end{equation*}
which can be appended to the state dynamics (preceding the construction of \eqref{eq:control_system_affine1} or \eqref{eq:control_system_affine2} above) so that the output map becomes linear in the (augmented) state. Thus for the purposes of characterizing complexity, we restrict our attention to control-affine systems with linear output maps of the form
\begin{equation}\label{eq:control_system}
\begin{split}
    \dot{x}(t) &= f(x(t)) + \sum_{i=1}^m u_i(t) g_i(x(t)),\quad\quad x(0) = x_0 \\
    y(t) &= c^\Tr x(t),
\end{split}
\end{equation}
where $f,g_1,\dots,g_m : \Reals^n \to \Reals^n$ and $c \in \Reals^n$. Lastly, we can disguise the drift if necessary by setting $g_0 \equiv f / M$ and $u_0 \equiv M$ for some constant $M \neq 0$, so without loss of generality we will only consider the driftless case (i.e., $f \equiv 0$) which will be convenient for certain calculations later.

\subsection{Chen--Fliess series}

To keep the paper self-contained, we give here a formal derivation of the Chen--Fliess series; see \citet{sussman_1983,isidori_1995,le_marbach_2023} for rigorous expositions, including the analysis of convergence and truncation errors. By the fundamental theorem of calculus, the output at time $t \geq 0$ can be written
\begin{equation}
\begin{split}\label{eq:output_expansion_1}
    y(t) &= c^\Tr x_0 + \int_0^t c^\Tr \dot{x}(s) \dif s \\
    &= c^\Tr x_0 + \int_0^t c^\Tr \left( \sum_{i=1}^m u_i(s) g_i(x(s)) \right) \dif s \\
    &= c^\Tr x_0 + \sum_{i=1}^m \int_0^t u_i(s) L_{g_i} c^\Tr x(s) \dif s.
\end{split}
\end{equation}
where $L_{g_i} c^\Tr x(s) = c^\Tr g_i(x(s))$ is the Lie derivative of the output function with respect to the vector field $g_i : \Reals^n \to \Reals^n$ at time $s \geq 0$. Using a similar trick, we can rewrite this Lie derivative as
\begin{align*}
    L_{g_i} c^\Tr x(s) &= L_{g_i} c^\Tr x_0 + \int_0^s c^\Tr \frac{\p g_i}{\p x}(x(r)) \dot{x}(r) \dif r \\
    &= L_{g_i} c^\Tr x_0 + \int_0^s c^\Tr \frac{\p g_i}{\p x}(x(r)) \left( \sum_{j=1}^m u_j(r) g_j(x(r)) \right) \dif r \\
    &= L_{g_i} c^\Tr x_0 + \sum_{j=1}^m \int_0^s u_j(r) L_{g_j} \circ L_{g_i} c^\Tr x(r) \dif r,
\end{align*}
where $\frac{\p g_i}{\p x}(x(r)) \in \Reals^{n \times n}$ is the Jacobian of $g_i$ evaluated at $x(r)$. Substituting this into \eqref{eq:output_expansion_1} gives
\begin{align}
    y(t) &= c^\Tr x_0 + \sum_{i=1}^m \int_0^t u_i(s) \left( L_{g_i} c^\Tr x_0 + \sum_{j=1}^m \int_0^s u_j(r) L_{g_j} \circ L_{g_i} c^\Tr x(r) \dif r \right) \dif s \label{eq:output_expansion_2} \\
    &= c^\Tr x_0 + \sum_{i=1}^m \left( \int_0^t u_i(s) \dif s \right) L_{g_i} c^\Tr x_0 + \sum_{1 \leq i,j \leq m} \int_0^t \int_0^s u_i(s) u_j(r) L_{g_j} \circ L_{g_i} c^\Tr x(r) \dif r \dif s \nonumber.
\end{align}
Repeating this process for $L_{g_j} \circ L_{g_i} c^\Tr x(r)$ in \eqref{eq:output_expansion_2} and for the resulting higher order Lie derivatives generates the so-called \textit{Chen--Fliess series}
\begin{equation*}
    y(t) = \sum_{\substack{1 \leq i_1,\dots,i_k \leq m\\k \geq 0}} \left( \int_0^t \int_0^{\tau_k} \cdots \int_0^{\tau_2} u_{i_k}(\tau_k) \cdots u_{i_1}(\tau_1) \dif \tau_1 \cdots \dif \tau_{k-1} \dif \tau_k \right) \left( L_{g_{i_1}} \circ \cdots \circ L_{g_{i_k}} c^\Tr x_0 \right).
\end{equation*}
It will be convenient later to have a more compact expression for this series. Denote the set of multi-indices by $\sW := \{ w = (i_1,\dots,i_k) : 1 \leq i_1,\dots,i_k \leq m,\ k \geq 0 \}$. For a multi-index $w = (i_1,\dots,i_k)$, let $L_w := L_{g_{i_1}} \circ \cdots \circ L_{g_{i_k}}$ and $u_w(\tau) := u_{i_1}(\tau_1) \cdots u_{i_k}(\tau_k)$. The region of integration is a $k$-simplex, which we denote by $\Delta^k(t) := \{(\tau_1,\dots,\tau_k) : 0 \leq \tau_1 \leq \cdots \leq \tau_k \leq t \}$. Now we can write
\begin{equation}\label{eq:chen_fliess_series}
    y(t) = \sum_{w \in \sW} \left( \int_{\Delta^{|w|}(t)} u_w(\tau) \dif \tau \right) L_w c^\Tr x_0
\end{equation}
where $|w|$ is the length of the multi-index $w$. The term corresponding to the empty multi-index is simply the constant $c^\Tr x_0$.

\subsection{Sequence-space embeddings}\label{sec:seq_embeddings}

Consider a space of bounded, measurable inputs $\cU \subset \{ u : [0,T] \to \Reals^m : u_i(t) \in [-M,M] \}$. We can embed this space of functions $\cU$ into the space of real-valued sequences $\Reals^\sW$ via the so-called \textit{signature} $S : \cU \to \Reals^\sW$. For each $w \in \sW$, define
\begin{equation*}
    S^w(u) := \int_{\Delta^{|w|}(T)} u_w(\tau) \dif \tau.
\end{equation*}
Observe that $|S^w(u)| \leq \frac{(MT)^{|w|}}{|w|!}$, since the integrand is bounded by $|u_w(\tau)| \leq M^{|w|}$ and the volume of the $|w|$-simplex $\Delta^{|w|}(T)$ is $\frac{T^{|w|}}{|w|!}$.

Now consider a compact set of initial conditions $\sX \subset \Reals^n$. In the same manner that $\cU$ is embedded into a sequence space, we can also embed $\sX$ into a sequence space via a map $\Phi : \sX \to \Reals^\sW$ that computes iterated Lie derivatives of the output map.
For each $w \in \sW$, define
\begin{equation*}
    \Phi^w(x) := L_w c^\Tr x.
\end{equation*}
The embeddings $S$ and $\Phi$ pair naturally to recover the Chen--Fliess series \eqref{eq:chen_fliess_series} concisely as
\begin{equation}\label{eq:chen_fliess_embeddings}
    y(T) = \left\langle S(u), \Phi(x_0) \right\rangle = \sum_{w \in \sW} S^w(u) \Phi^w(x_0).
\end{equation}
This representation of the output $y(T)$ admits a natural interpretation as a linear combination of nonlinear ``features'' $\Phi^w(x_0)$, where the ``weights'' are precisely the signature elements $S^w(u)$, which is a well-understood model architecture in learning theory. It is important to recognize that this (formal) series may fail to converge unless the time horizon $T$ is sufficiently short, the control magnitude $M$ is sufficiently small, and/or certain regularity assumptions on the vector fields $g_1,\dots,g_m$ are satisfied. Such conditions are needed so that the map $x_0 \mapsto y(T)$ is well-defined. In a later section we will consider sufficient assumptions to guarantee convergence.

\section{Main result}\label{sec:main_result}

Suppose we have a sample of i.i.d.\ random vectors $(X_1,Y_1),\dots,(X_N,Y_N)$ drawn according to a probability measure $\mu$ with compact support $\mathrm{supp}(\mu) = \sX \times \sY \subset \Reals^n \times \Reals$. We seek to identify a function $\varphi : \sX \to \sY$ that approximately reproduces this sample and generalizes to other identically distributed samples. Consider a class of such functions $\cF$ given by
\begin{equation*}
    \cF = \left\{ x \mapsto \varphi(x) = \left\langle S(u), \Phi(x) \right\rangle : u \in \cU \right\}.
\end{equation*}
The vector fields $g_1,\dots,g_m$ implicit in the definition of $\Phi$ are considered to be fixed, and the learnable parameters are represented by the control input $u$ (or equivalently, the elements of the signature $S(u)$). Given a loss function $\ell : \sY \times \sY \to \Reals$, define the \textit{expected risk}
\begin{equation*}
    L_\mu(\varphi) := \Ex_\mu \left[ \ell(Y,\varphi(X)) \right] = \int_{\sX \times \sY} \ell(y,\varphi(x)) \mu(\dif x,\dif y),
\end{equation*}
the \textit{minimum risk} $L_\mu^*(\cF) := \inf_{\varphi \in \cF} L_\mu(\varphi)$, and the \textit{empirical risk} $\frac{1}{N} \sum_{i=1}^N \ell(Y_i,\varphi(X_i))$. The empirical risk minimization (ERM) algorithm can be stated succinctly as
\begin{equation*}
    \hat{\varphi} \in \argmin_{\varphi \in \cF} \frac{1}{N} \sum_{i=1}^N \ell(Y_i, \varphi(X_i)),
\end{equation*}
and this minimization problem is usually solved numerically using e.g., gradient descent or its variations. Assuming that $0 \leq \ell(y, \varphi(x)) \leq B$ for all $(x,y) \in \sX \times \sY$, $\varphi \in \cF$, then the following excess risk guarantee holds with probably at least $1 - \delta$ (see e.g., \citet{hajek_raginsky_2021}):
\begin{equation}\label{eq:excess_risk_guarantee}
    L_\mu(\hat{\varphi}) - L_\mu^*(\cF) \leq 4 \Ex \cR_N(\ell \circ \cF) + B \sqrt{\frac{2 \log\big( \frac{1}{\delta} \big)}{N}},
\end{equation}
where the quantity $\cR_N(\ell \circ \cF)$ is the empirical Rademacher complexity conditioned on the data. This is given by
\begin{equation}\label{eq:rademacher_complexity}
    \cR_N(\ell \circ \cF) := \Ex_\epsilon \left[ \sup_{\varphi \in \cF} \frac{1}{N} \left| \sum_{i=1}^N \epsilon_i \ell(Y_i, \varphi(X_i))\right|\right],
\end{equation}
where the expectation is taken with respect to the sequence $\epsilon_1,\dots,\epsilon_N$ of i.i.d.\ Rademacher random variables that are independent of the data. We can see that if the right-hand side of \eqref{eq:excess_risk_guarantee} is small, then the expected risk of the ERM map $\hat{\varphi}$ is close to the minimum risk, in which case we would say that $\hat{\varphi}$ generalizes well. Hence to study generalizability of $\cF$, we seek to bound $\cR_N(\ell \circ \cF)$.

If the loss function $\ell$ is well-behaved, we can often bound $\cR_N(\ell \circ \cF)$ directly in terms of $\cR_N(\cF)$. For instance, let $\ell(y,\varphi(x)) = (y - \varphi(x))^2$ and assume that $\sup_{y \in \sY} |y| \leq M_1$ and $\sup_{x \in \sX} \sup_{\varphi \in \cF} |\varphi(x)| \leq M_2$. Then by observing that $\left(y - \varphi(x)\right) \mapsto \ell(y,\varphi(x))$ is $2(M_1 + M_2)$-Lipschitz and using the contraction principle \citep{ledoux_talagrand_1991}, we have
\begin{align*}
    \cR_N(\ell \circ \cF) &\leq 4 \left( M_1 + M_2 \right) \left(\frac{M_1}{\sqrt{N}} + \cR_N(\cF) \right).
\end{align*}
If instead the loss is given by $\ell(y,\varphi(x)) = |y - \varphi(x)|$, which is 1-Lipschitz as a function of $\left(y - \varphi(x)\right)$, then using the contraction principle gives
\begin{align*}
    \cR_N(\ell \circ \cF) &\leq \frac{2M_1}{\sqrt{N}} + 2 \cR_N(\cF).
\end{align*}
With this in mind, for the remainder of this section we will focus on bounding $\cR_N(\cF)$.
\begin{theorem}\label{thm:main_result}
The empirical Rademacher complexity of $\cF$ is bounded by
\begin{equation}\label{eq:main_theorem}
    \cR_N(\cF) \leq \frac{1}{\sqrt{N}} \sum_{k \geq 0} \frac{\left( mMT \right)^{k}}{k!} \Lambda_k,
\end{equation}
where $\Lambda_k := \sup \left\{ \left| L_w c^\Tr x \right| : x \in \sX,\ w \in \sW,\ |w| = k\right\}$.
\end{theorem}

\noindent In the proof of Theorem \ref{thm:main_result} we will use the following lemma:
\begin{lemma}\label{lem:jensen}
    Let $\psi : \sX \to \Reals$ be an arbitrary function. Then
    \begin{equation}\label{eq:jensen_lemma}
        \Ex_\epsilon \left[ \left| \sum_{i=1}^N \epsilon_i \psi(X_i) \right| \right] \leq \sqrt{N} \sup_{x \in \sX} |\psi(x)|.
    \end{equation}
\end{lemma}

\begin{proof} Jensen's inequality states that for a concave function $f : \Reals \to \Reals$ and a real-valued random variable $X$, we have $\Ex[f(X)] \leq f(\Ex[X])$. Recall that $x \mapsto \sqrt{x}$ is concave. Then by linearity of expectation and mutual independence of $\epsilon_1,\dots,\epsilon_N$ we have
    \begin{align*}
        \Ex_\epsilon \left[ \left| \sum_{i=1}^N \epsilon_i \psi(X_i) \right| \right] &= \Ex_\epsilon \left[ \sqrt{ \left( \sum_{i=1}^N \epsilon_i \psi(X_i) \right) \left( \sum_{j=1}^N \epsilon_j \psi(X_j) \right) } \right] \\
        &= \Ex_\epsilon \left[ \sqrt{ \sum_{i,j=1}^N \epsilon_i \epsilon_j \psi(X_i) \psi(X_j) } \right] \\
        &\leq \sqrt{ \sum_{i,j=1}^N \Ex_\epsilon \left[ \epsilon_i \epsilon_j \right] \psi(X_i), \psi(X_j) } \\
        &= \sqrt{ \sum_{i=1}^N \left| \psi(X_i) \right|^2 } \leq \sqrt{N} \sup_{x \in \sX} |\psi(x)|.
    \end{align*}
\end{proof}

\noindent Now we will prove Theorem \ref{thm:main_result}.
\begin{proof}
The Rademacher complexity of $\cF$ is defined by
\begin{equation*}
    \cR_N(\cF) = \frac{1}{N} \Ex_\epsilon \left[ \sup_{\varphi \in \cF} \left| \sum_{i=1}^N \epsilon_i \varphi(X_i) \right| \right] = \frac{1}{N} \Ex_\epsilon \left[ \sup_{u \in \cU} \left| \sum_{i=1}^N \epsilon_i \sum_{w \in \sW} S^w(u) \Phi^w(X_i) \right| \right].
\end{equation*}
We can bound the expression inside the expectation using the triangle inequality:
\begin{align*}
    \sup_{u \in \cU} \left| \sum_{i=1}^N \epsilon_i \sum_{w \in \sW} S^w(u) \Phi^w(X_i) \right| &\leq \sup_{u \in \cU} \sum_{w \in \sW} \left|  S^w(u) \right| \left| \sum_{i=1}^N \epsilon_i \Phi^w(X_i) \right| \\
    &\leq \sum_{w \in \sW} \frac{(MT)^{|w|}}{|w|!} \left| \sum_{i=1}^N \epsilon_i L_w c^\Tr X_i \right|.
\end{align*}
Applying Lemma \ref{lem:jensen} with $\psi \gets L_w c^\Tr$ gives 
\begin{equation*}
    \Ex_\epsilon \left[ \left| \sum_{i=1}^N \epsilon_i L_w c^\Tr X_i \right| \right] \leq \sqrt{N} \sup_{x \in \sX} |L_w c^\Tr x|.
\end{equation*}
Putting everything together, we have
\begin{align*}
    \cR_N(\cF) &\leq \frac{1}{N} \Ex_\epsilon \left[ \sup_{u \in \cU} \left| \sum_{i=1}^N \epsilon_i \sum_{w \in \sW} S^w(u) \Phi^w(X_i) \rangle \right| \right] \\
    &\leq \frac{1}{N} \Ex_\epsilon \left[ \sum_{w \in \sW} \frac{(MT)^{|w|}}{|w|!} \left| \sum_{i=1}^N \epsilon_i L_w c^\Tr X_i \right| \right] \\
    &\leq \frac{1}{N}\sum_{w \in \sW} \frac{(MT)^{|w|}}{|w|!} \Ex_\epsilon \left[ \left| \sum_{i=1}^N \epsilon_i L_w c^\Tr X_i \right| \right] \\
    &\leq \frac{1}{\sqrt{N}} \sum_{w \in \sW} \frac{(MT)^{|w|}}{|w|!} \sup_{x \in \sX} \left| L_w c^\Tr x \right| \\
    &\leq \frac{1}{\sqrt{N}} \sum_{k \geq 0} \frac{\left( mMT \right)^{k}}{k!} \Lambda_k.
\end{align*}
We are able to exchange the order of the infinite sum and the expectation above using Tonelli's theorem, because the summand/integrand are uniformly non-negative and measurable, and the rest follows from definitions.
\end{proof}

Now instantiating the Rademacher complexity bound comes down to bounding the norm of iterated Lie derivatives $\Lambda_k$, which we explore in some examples in the following section.

\section{Examples}\label{sec:examples}

In the following examples, let $r := \sup_{x \in \sX} |x|$ and assume $|c^\Tr| = 1$. The output is always given by $y(t) = c^\Tr x(t)$. We continue to use $M,T$ as they are as defined in Section \ref{sec:seq_embeddings}. Recall that some conditions are needed on $M,T$ and/or the vector fields $g_1,\dots,g_m$ so that the series \eqref{eq:chen_fliess_embeddings} converges and the map $x_0 \mapsto y(T)$ is well-defined. Looking at error estimates for the Chen--Fliess expansion \citep{le_marbach_2023}, it appears natural to suggest a condition like $\Lambda_k \leq C^k k!$ for some constant $C > 0$ depending on $M,T,g_1,\dots,g_m$, which yields a geometric series. However for some systems, this condition is overly restrictive and we can in fact achieve convergence for any $M,T$ even if this condition is violated. On the other hand, this condition excludes certain systems of interest that still admit convergent series expansions for some $M,T$ which we will see in the following examples.

\paragraph{Example~1} Consider the class of bilinear systems
\begin{align*}
    \dot{x}(t) &= \left(\sum_{i=1}^m A_i u_i(t) \right) x(t),\quad\quad x(0) = x_0 
\end{align*}
where $A_1,\dots,A_m \in \Reals^{n \times n}$. Let $a := \max_{i=1,\dots,m} \sigma_{\max} (A_i)$ be the maximum spectral norm of the matrices $A_1,\dots,A_m$. The Lie derivative of a linear function with respect to a linear vector field is simple to compute, leading to the following bound:
\begin{align*}
    \Lambda_k &= \sup \left\{ \left| L_w c^\Tr x \right| : x \in \sX,\ w \in \sW,\ |w| = k \right\} \\
    &= \sup \left\{ \left| c^\Tr A_{i_1} \cdots A_{i_k} x \right| : x \in \sX,\ w = (i_1,\dots,i_k) \right\} \\
    &\leq r \max_{i_1,\dots,i_k} \|A_{i_1}\| \cdots \|A_{i_k}\| \leq r a^k.
\end{align*}
Substituting this into Theorem \ref{thm:main_result} yields
\begin{align*}
    \cR_N(\cF) &\leq \frac{1}{\sqrt{N}} \sum_{k \geq 0} \frac{\left( mMT \right)^{k}}{k!} \Lambda_k \\
    &\leq \frac{1}{\sqrt{N}} \sum_{k \geq 0} \frac{\left( mMT \right)^{k}}{k!} r a^k \\
    &= \frac{r}{\sqrt{N}} \exp\left( mMTa \right),
\end{align*}
which is defined for all $M$, $T$.

\paragraph{Example~2} Consider the class of control-affine systems
\begin{align*}
    \dot{x}(t) &= \sum_{i=1}^m u_i(t) g_i(x(t)),\quad\quad x(0) = x_0 
\end{align*}
where $g_1,\dots,g_m : \Reals^n \to \Reals^n$ are analytic vector fields. Let $\tilde{g}_1,\dots,\tilde{g}_m : \Complex^n \to \Complex^n$ represent analytic continuations of $g_1,\dots,g_m$ and let $\tilde{L}_w := L_{\tilde{g}_{i_1}} \circ \cdots \circ L_{\tilde{g}_{i_k}}$. Denote a closed polydisc by
\begin{equation*}
    P(\xi,\rho) := \{(z_1,\dots,z_n) \in \Complex^n : |z_i-\xi_i| \leq \rho,\, 1 \le i \le n\}.
\end{equation*}
It is evident that $\iota(\sX) \subset P(\iota(x),2r)$ for any $x \in \sX$, where $\iota: \Reals^n \xhookrightarrow{} \Complex^n$ is the inclusion map. Define the component-wise maximum modulus of the complex vector fields $\tilde{g}_1,\dots,\tilde{g}_m$ by
\begin{equation*}
    a(r) := \max_{i=1,\dots,m} \max_{j=1,\dots,n} \sup_{x \in \sX} \sup_{z \in P(\iota(x),2r)} |\tilde{g}_i^j(z)|,
\end{equation*}
where $\tilde{g}_i^j : \Complex^n \to \Complex$ is the $j$th component of $\tilde{g}_i : \Complex^n \to \Complex^n$. We can apply Lemma 3.8 in \citet{lesiak_krener_1978}, which is based on Cauchy estimates from complex analysis, to bound $\Lambda_k$ as follows:
\begin{align*}
    \Lambda_k &= \sup \left\{ \left| L_w c^\Tr x \right| : x \in \sX,\ w \in \sW,\ |w| = k \right\} \\
    &\leq \sup \left\{ \left| \tilde{L}_w c^\Tr \iota(x) \right| : x \in \sX,\ z \in P(\iota(x),2r),\ w \in \sW,\ |w| = k \right\} \\
    &= \sup \left\{ \left| \left( \sum_{j=1}^n \tilde{g}_{i_k}^j(z) \frac{\p}{\p z_j} \right) \cdots \left( \sum_{j=1}^n \tilde{g}_{i_1}^j(z) \frac{\p}{\p z_j} \right) c^\Tr \iota(x) \right| : x \in \sX,\ z \in P(\iota(x),2r),\ 1 \leq i_1,\dots,i_k \leq m \right\} \\
    &\leq \left( \frac{n a(r)}{2r} \right)^k \left( k! 2^{(n+1)k} \right) \left(1 + 2\sqrt{n}\right) r \\
    &= k! \left(\frac{2^n n a(r)}{r} \right)^{k} \left(1 + 2\sqrt{n}\right) r
\end{align*}
where we have used that for $z \in P(\iota(x),2r)$,
\begin{equation*}
    \left| c^\Tr z \right| \leq \left| c \right| \left| z \right| \leq \left| c \right| \left( \left| \iota(x) \right| + \sqrt{(2r)^2 n} \right) \leq \left(1 + 2\sqrt{n}\right) r.
\end{equation*}
Assuming that $2^n n m M T a(r) < r$, substituting this into Theorem \ref{thm:main_result} yields
\begin{align*}
    \cR_N(\cF) &\leq \frac{1}{\sqrt{N}} \sum_{k \geq 0} \frac{\left( mMT \right)^{k}}{k!} \Lambda_k \\
    &\leq \frac{\left(1 + 2\sqrt{n}\right) r}{\sqrt{N}} \sum_{k \geq 0} \frac{\left( mMT \right)^{k}}{k!} k! \left(\frac{2^n n a(r)}{r} \right)^{k} \\
    &= \frac{\left(1 + 2\sqrt{n}\right) r}{\sqrt{N}} \sum_{k \geq 0} \left(\frac{2^n n mMT a(r)}{r}\right)^{k} \\
    &= \frac{\left(1 + 2\sqrt{n}\right) r}{\sqrt{N}} \frac{r}{r-2^n n mMT a(r)}.
\end{align*}

\paragraph{Example~3}
Consider a class of Hopfield nets
\begin{align*}
    \dot{x}(t) = u(t) \sigma(x(t)) = \sum_{1 \leq i,j \leq n} u_{ij}(t) \sigma(x_j(t)) e_i,
\end{align*}
where $e_i \in \Reals^n$ is the $i$th unit vector, $u : [0,T] \to \Reals^{n \times n}$ is a matrix-valued control input, and $\sigma : \Reals \to \Reals$ is a sigmoidal nonlinearity. Suppose the derivatives of $\sigma$ satisfy the bound $\sup_{x \in \Reals} |\sigma^{(k)}(x)| \leq b a^{k} k!$ for some $a,b > 0$, which holds for many common sigmoidal activation functions. If $k=0$, then $\Lambda_k \leq r$, so suppose $k \geq 1$. Then

\begin{align*}
    \Lambda_k &= \sup \left\{ \left| L_w c^\Tr x \right| : x \in \sX,\ w \in \sW,\ |w| = k \right\} \\
    &= \sup \left\{ \left| \left( \sigma(x_{j_1}) e_{i_1}^\Tr \frac{\p}{\p x} \right) \cdots \left( \sigma(x_{j_k}) e_{i_k}^\Tr \frac{\p}{\p x} \right) c^\Tr x \right| : x \in \sX,\ 1 \leq i_1, j_1,\dots,i_k, j_k \leq n \right\} \\
    &= \sup \left\{ \left| \left( \sigma(x_{j_1}) \frac{\p}{\p x_{i_1}} \right) \cdots \left( \sigma(x_{j_k}) \frac{\p}{\p x_{i_k}} \right) c^\Tr x \right| : x \in \sX,\ 1 \leq i_1, j_1,\dots,i_k, j_k \leq n \right\} \\
    &\leq \sup \left\{ \binom{k}{n_1,\dots,n_k} \left| \sigma^{(n_1)}(x) \cdots \sigma^{(n_k)}(x) \right| : x \in [-r,r],\ n_1+\cdots+n_k=k-1\right\} \\
    &\leq \sup \left\{ \binom{k}{n_1,\dots,n_k} (ba^{n_1} n_1!) \cdots (ba^{n_k} n_k!) : n_1+\cdots+n_k=k-1\right\} \\
    &\leq \gamma(k) b^k a^{k-1},
\end{align*}
where $\gamma(k) = \frac{k!}{2} \binom{2k}{k}$, which comes from a counting argument; see Appendix B.5 in \citet{fermanian_marion_2021}. Assuming that $4 n^2 MT ba < 1$, substituting this into Theorem \ref{thm:main_result} yields
\begin{align*}
    \cR_N(\cF) &\leq \frac{1}{\sqrt{N}} \sum_{k \geq 0} \frac{\left( n^2 MT \right)^{k}}{k!} \Lambda_k \\
    &\leq \frac{1}{\sqrt{N}} \left(r + \sum_{k \geq 1} \frac{\left( n^2 MT \right)^{k}}{k!} \frac{k!}{2} \binom{2k}{k} b^k a^{k-1} \right) \\
    &= \frac{1}{\sqrt{N}} \left(r + \frac{1}{2a}\sum_{k \geq 1} \binom{2k}{k} \left(n^2 MT ba \right)^{k} \right) \\
    &= \frac{1}{\sqrt{N}} \left(r - \frac{1}{2a} + \frac{1}{2a\sqrt{1-4 n^2 MT ba}} \right).
\end{align*}
The final expression above comes from the generating function for the central binomial coefficients
\begin{equation*}
    \frac{1}{\sqrt{1-4x}} = \sum_{k \geq 0} \binom{2k}{k} x^k,
\end{equation*}
which can be derived by applying the generalized binomial theorem with $n = -\frac{1}{2}$.

\section{Conclusion}\label{sec:conclusion}

Using the Chen--Fliess series for nonlinear ODEs, we have shown how continuous-depth nets (a.k.a. neural ODEs) can be viewed as a kind of single-layer, infinite-width net, where the ``weights'' are the iterated integrals (signature elements) of the control input, and the ``features'' are the iterated Lie derivatives of the output function. This approach facilitates compact expressions for the generalization performance of ODE models based on the comparatively simpler analysis of single-layer architectures. These bounds are also straightforward to instantiate given various assumptions about the structure of the ODE model, which we have demonstrated through some examples.

One barrier to applying this technique in more generality is that the Chen--Fliess series converges only for sufficiently small time horizons and/or small control magnitudes. One could attempt to circumvent this issue by dividing the time horizon into slices and considering the composition of several convergent series expansions of the flow map, possibly later taking the limit as the number of slices increases to infinity. However, this sacrifices the convenience of working with a single-layer architecture, as bounding the Rademacher complexity of composite function classes is typically either challenging or yields conservative results. An alternative approach to generalize the main result is to interpret the control input as a perturbation of some nominal control trajectory, and instead focus on obtaining margin bounds, which is an interesting direction for follow-up work.

Incorporating any special structure known about the system of interest would also likely give sharper results in cases where it applies. For example, the result here is agnostic to any information concerning system stability or dissipativity. Instead of bounding the Lie derivatives of the vector fields directly, one could likely obtain tighter bounds in the stable case by using the logarithmic norm of the iterated Jacobians of the vector fields instead of the operator norm, for instance, or otherwise specializing the analysis to incorporate any behavioral or structural knowledge of the ODE model under consideration.

\section*{Acknowledgements}

This work was supported in part by the NSF under award CCF-2106358 (``Analysis and Geometry of Neural Dynamical Systems'') and in part by the Illinois Institute for Data Science and Dynamical Systems (iDS${}^2$), an NSF HDR TRIPODS institute, under award CCF-1934986.

\bibliography{references}

\end{document}